\documentclass[11pt,a4paper]{book}

\usepackage[utf8]{inputenc}
\usepackage[T1]{fontenc}
\usepackage{lmodern}

\usepackage{geometry}
\usepackage{setspace}
\usepackage{amsmath, amssymb, amsthm, mathtools}
\usepackage{thmtools}
\usepackage{bm}           

\usepackage{booktabs}
\usepackage{tikz}
\usepackage{graphicx}
\graphicspath{{figures/}}

\usepackage[numbers,sort&compress]{natbib} 
\usepackage{hyperref}
\usepackage[nameinlink]{cleveref}

\usepackage{enumitem}

\usepackage{xcolor}

\declaretheorem[style=plain,numberwithin=chapter]{theorem}

\declaretheorem[sibling=theorem,style=plain]{proposition}

\declaretheorem[sibling=theorem,style=definition]{definition}
\declaretheorem[sibling=theorem,style=definition]{example}
\declaretheorem[sibling=theorem,style=remark]{remark}

\usepackage[acronym]{glossaries}
\newglossary[not]{notation}{ntn}{ndn}{Notation}
\makeglossaries

\newglossarystyle{alignednotation}{
  \setglossarystyle{long}%
}
\newglossaryentry{naturals}{
  type=notation,
  name={\ensuremath{\mathbb{N}}},
  description={Set of natural numbers},
  sort=naturals
}
\newcommand{\N}{\ensuremath{\mathbb{N}}}
\newglossaryentry{reals}{
  type=notation,
  name={\ensuremath{\mathbb{R}}},
  description={Set of real numbers},
  sort=reals
}
\newcommand{\R}{\ensuremath{\mathbb{R}}}
\newglossaryentry{complex}{
  type=notation,
  name={\ensuremath{\mathbb{C}}},
  description={Set of complex numbers},
  sort=complex
}
\newcommand{\C}{\ensuremath{\mathbb{C}}}
\newglossaryentry{omega}{
  type=notation,
  name={\ensuremath{\Omega}},
  description={Sample space: the set of all possible outcomes of an experiment},
  sort=omega
}
\newcommand{\SampleSpace}{\ensuremath{\Omega}}

\newglossaryentry{sigma-algebra}{
  type=notation,
  name={\ensuremath{\mathcal{F}}},
  description={$\sigma$-algebra of measurable subsets of $\Omega$, closed under complements and countable unions},
  sort=sigmaalgebra
}
\newcommand{\SigmaAlg}{\ensuremath{\mathcal{F}}}

\newglossaryentry{probmeasure}{
  type=notation,
  name={\ensuremath{P}},
  description={Probability measure $P:\mathcal{F}\to[0,1]$},
  sort=probmeasure
}
\newcommand{\ProbMeasure}{\ensuremath{P}}

\newglossaryentry{borel}{
  type=notation,
  name={\ensuremath{\mathcal{B}(\R)}},
  description={Borel $\sigma$-algebra on $\R$, generated by open intervals.},
  sort=borel
}
\newcommand{\Borel}{\ensuremath{\mathcal{B}(\R)}}
\newglossaryentry{powerset}{
  type=notation,
  name={\ensuremath{2^{\SampleSpace}}},
  description={\emph{Power set} of $\SampleSpace$: the set of all subsets of $\SampleSpace$},
  sort=powerset
}

\newglossaryentry{law}{
  type=notation,
  name={\ensuremath{P_X}},
  description={Distribution (law) of a random variable $X$, defined by $P_X(B)=P(X^{-1}(B))$},
  sort=law
}
\newcommand{\Dist}[1]{\ensuremath{P_{#1}}}

\newglossaryentry{expectation}{
  type=notation,
  name={\ensuremath{\mathbb{E}}},
  description={Expectation (mean) operator. $\E{X}=\int x\,dP_X(x)$},
  sort=expectation
}

\newcommand{\E}[2][]{\mathbb{E}_{#1}\!\left[#2\right]}

\newglossaryentry{variance}{
  type=notation,
  name={\ensuremath{\mathrm{Var}}},
  description={Variance operator. $\Var{X}=\E{(X-\E{X})^2}$},
  sort=variance
}
\newcommand{\VarOp}{\ensuremath{\mathrm{Var}}}
\newcommand{\Var}[1]{\VarOp\!\left(#1\right)}

\newglossaryentry{covariance}{
  type=notation,
  name={\ensuremath{\mathrm{Cov}}},
  description={Covariance operator. $\Cov{X}{Y}=\E{(X-\E{X})(Y-\E{Y})}$},
  sort=covariance
}
\newcommand{\CovOp}{\ensuremath{\mathrm{Cov}}}
\newcommand{\Cov}[2]{\CovOp\!\left(#1,#2\right)}

\newcommand{\Prob}[1]{\ProbMeasure\!\left(#1\right)}

\newcommand{\RV}[1]{\ensuremath{\bm{#1}}}
\newacronym{rv}{RV}{random variable}
\newacronym{pmf}{PMF}{probability mass function}
\newacronym{pdf}{PDF}{probability density function}
\newacronym{cdf}{CDF}{cumulative distribution function}
\newacronym{iid}{i.i.d.}{independent and identically distributed}
\newacronym{rkhs}{RKHS}{Reproducing Kernel Hilbert Space}
\newacronym{kl}{KL}{Kullback--Leibler divergence}

\title{Notes on Kernel Methods}
\author{
  Diego A. Pérez-Rosero, Danna V. Salazar-Dubois, \\ Juan C. Lugo-Rojas, and Andrés M. Álvarez-Meza, \\ German Castellanos-Dominguez \\[0.5em]
  \small Signal Processing and Recognition Group, Universidad Nacional de Colombia
}
\date{\today}

\begin{document}

\maketitle
\tableofcontents

\chapter{Mathematical and Statistical Preliminaries}
\label{ch:pre}

\section{Probability Spaces and Random Variables}
\label{sec:probability}

To rigorously describe random phenomena, we begin with the concept of a 
\emph{probability space}, which formalizes the notion of uncertainty and provides
the mathematical foundation for probability theory~\cite{billingsley1995probability,gut2013probability,wasserman2004all}.
This framework specifies which outcomes are possible, which subsets of outcomes
are measurable, and how numerical probabilities are consistently assigned.

\begin{definition}[Sample Space]
The \emph{sample space} is a set
\[
\SampleSpace = \{\, \omega : \text{possible outcomes of an experiment} \,\}.
\]
Each element $\omega \in \SampleSpace$ represents an \emph{elementary outcome}.
The sample space constitutes the domain on which events and \glspl{rv} are defined.
\end{definition}

\begin{example}
For a single toss of a fair coin, $\SampleSpace = \{\text{H}, \text{T}\}$.
For two coin tosses, $\SampleSpace = \{\text{HH}, \text{HT}, \text{TH}, \text{TT}\}$.
If we observe a particle’s position in $\R^3$, then $\SampleSpace = \R^3$.
\end{example}

\begin{definition}[$\sigma$-Algebra]
A \emph{$\sigma$-algebra} on $\SampleSpace$ is a collection
\(\gls{sigma-algebra} \subseteq \gls{powerset}\)
of subsets of $\SampleSpace$ satisfying:
\begin{enumerate}
    \item $\SampleSpace \in \SigmaAlg$;
    \item If $A \in \SigmaAlg$, then its complement $A^{\mathrm{c}} = \SampleSpace \setminus A$ also lies in $\SigmaAlg$;
    \item If $\{A_i\}_{i=1}^{\infty} \subseteq \SigmaAlg$, then $\bigcup_{i=1}^{\infty} A_i \in \SigmaAlg$.
\end{enumerate}
These axioms imply closure under countable intersections by De Morgan’s laws.
The pair $(\SampleSpace,\SigmaAlg)$ is called a \emph{measurable space}.
\end{definition}

\begin{remark}[Examples of $\sigma$-Algebras]
A $\sigma$-algebra specifies which subsets of $\SampleSpace$ are regarded as
\emph{measurable events}, i.e., events to which probability can be assigned.
We now present several illustrative examples, ordered from simplest to most commonly used.

\begin{example}[Trivial $\sigma$-Algebra]
For any sample space $\SampleSpace$, the collection
\[
    \Sigma_{\mathrm{triv}}
    = \{\,\emptyset,\;\SampleSpace\,\}
\]
is a $\sigma$-algebra.
Indeed:
\begin{itemize}
    \item it contains $\SampleSpace$;
    \item the complement of either set is still in the collection;
    \item any countable union of sets in the collection equals either
          $\emptyset$ or $\SampleSpace$.
\end{itemize}
It is the \emph{smallest} possible $\sigma$-algebra on $\SampleSpace$.
\end{example}

\begin{example}[Power Set]
The power set $\gls{powerset} = 2^{\SampleSpace}$, consisting of
\emph{all} subsets of $\SampleSpace$, is also a $\sigma$-algebra:
\begin{itemize}
    \item it contains $\SampleSpace$ by definition;
    \item complements and countable unions of subsets of $\SampleSpace$
          are still subsets of $\SampleSpace$.
\end{itemize}
Thus $2^{\SampleSpace}$ is the \emph{largest} $\sigma$-algebra on $\SampleSpace$.
This is typically used when $\SampleSpace$ is finite or countable.
\end{example}

\begin{example}[Borel $\sigma$-Algebra on $\R$]
When $\SampleSpace = \R$, one commonly works with the \emph{Borel $\sigma$-algebra},
denoted \gls{borel}.
It is defined as the \emph{collection of all sets that can be constructed
from open intervals using countable unions, countable intersections,
and complements}.
For instance:
\begin{itemize}
    \item every open interval $(a,b)$ belongs to $\Borel$;
    \item therefore every closed interval $[a,b]$, every half-line $(-\infty,a]$,
          and every finite union of intervals belongs to $\Borel$;
    \item more complicated sets formed by repeatedly applying countable unions,
          intersections, and complements also belong to $\Borel$.
\end{itemize}
The Borel $\sigma$-algebra provides the standard notion of measurability
for real-valued \glspl{rv}~\cite{billingsley1995probability,gut2013probability}.
\end{example}

\end{remark}

\begin{definition}[Probability Measure]
A \emph{probability measure} on $(\SampleSpace,\SigmaAlg)$
is a function
\[
\ProbMeasure : \SigmaAlg \longrightarrow [0,1]
\]
satisfying:
\begin{enumerate}
    \item \textbf{Non-negativity:} $\ProbMeasure(A) \ge 0$ for all $A \in \SigmaAlg$;
    \item \textbf{Normalization:} $\ProbMeasure(\SampleSpace) = 1$;
    \item \textbf{Countable additivity:} For any countable collection of pairwise disjoint events
    $\{A_i\}_{i=1}^{\infty} \subseteq \SigmaAlg$,
    \[
        \ProbMeasure\!\left(\bigcup_{i=1}^{\infty} A_i\right)
        = \sum_{i=1}^{\infty} \ProbMeasure(A_i).
    \]
\end{enumerate}
The triplet $(\SampleSpace,\SigmaAlg,\ProbMeasure)$ is called a \emph{probability space}.
\end{definition}

\begin{remark}[Basic Consequences]
From these axioms follow the fundamental properties of measures:
\begin{itemize}
    \item $\ProbMeasure(\emptyset) = 0$;
    \item \emph{Monotonicity:} if $A \subseteq B$, then $\ProbMeasure(A) \le \ProbMeasure(B)$;
    \item \emph{Finite additivity:} if $A,B \in \SigmaAlg$ are disjoint, 
          then $\ProbMeasure(A \cup B)=\ProbMeasure(A)+\ProbMeasure(B)$;
    \item \emph{Continuity from below:} if $A_n \uparrow A$, then $\ProbMeasure(A_n) \uparrow \ProbMeasure(A)$;
    \item \emph{Continuity from above:} if $A_n \downarrow A$, then $\ProbMeasure(A_n) \downarrow \ProbMeasure(A)$.
\end{itemize}
These properties ensure consistency: probabilities are additive,
non-negative, and behave continuously under limits of increasing or decreasing event sequences~\cite{billingsley1995probability,gut2013probability}.
\end{remark}

\begin{definition}[Random Variable]
\label{def:random-variable}
Let $(\SampleSpace,\SigmaAlg,\ProbMeasure)$ be a probability space,
and let $(\R,\Borel)$ denote the measurable space of real numbers with the Borel $\sigma$-algebra.
A \emph{\acrfull{rv}} is a function
\[
X : \SampleSpace \longrightarrow \R
\]
that is \emph{measurable} with respect to $\SigmaAlg$ and $\Borel$; that is,
for every Borel set $B \in \Borel$~\cite{billingsley1995probability,gut2013probability},
\[
X^{-1}(B) = \{\,\omega \in \SampleSpace : X(\omega) \in B\,\} \in \SigmaAlg.
\]
\end{definition}

\begin{remark}
The measurability condition ensures that the image of measurable sets under $X$
corresponds to measurable events in $\SampleSpace$.
Hence, probabilities involving $X$---such as $\Prob{X \le a}$ or $\Prob{a < X \le b}$---are well-defined.
Formally, the probability measure $\ProbMeasure$ on $(\SampleSpace,\SigmaAlg)$
induces a new measure $\Dist{X}$ on $(\R,\Borel)$, called the
\emph{distribution} or \emph{law} of $X$, defined by
\[
\Dist{X}(B) = \Prob{X^{-1}(B)}, \qquad B \in \Borel.
\]
Thus, the random variable $X$ transfers the probability structure from $\SampleSpace$
to the real line, allowing us to compute probabilities and expectations directly in $\R$~\cite{billingsley1995probability}.
\end{remark}

\begin{definition}[Discrete and Continuous Random Variables]
A \gls{rv} $X$ is said to be:
\begin{itemize}
    \item \emph{Discrete} if its range (the set of values it may take) is countable,
          and its \gls{pmf} is defined by
          \[
              p(x) = \Prob{X = x}, \qquad x \in \R,
          \]
          satisfying $\sum_x p(x) = 1$.
    \item \emph{Continuous} if there exists a nonnegative function $p : \R \to [0,\infty)$ such that
          for all $a,b \in \R$,
          \[
              \Prob{a < X \le b}
              = \int_{a}^{b} p(t)\,dt,
              \qquad
              \int_{\R} p(x)\,dx = 1.
          \]
          In this case $p$ is called the \gls{pdf} of $X$.
\end{itemize}
\end{definition}

\begin{remark}
Both discrete and continuous \glspl{rv} fall under the same measure-theoretic framework:
the distribution $\Dist{X}$ is a probability measure on $(\R,\Borel)$.
A discrete distribution assigns positive probability only to a countable set of points,
while a continuous one is characterized by the existence of a density function $p$
such that probabilities of intervals can be obtained by integration.
In practice, many real-world distributions exhibit both discrete and continuous components~\cite{billingsley1995probability,wasserman2004all}.
\end{remark}

\section{Expectation, Variance, and Covariance}
\label{sec:expectation-variance}

In probability theory, the most fundamental numerical summaries of a random variable
are its \emph{expectation} and \emph{variance}.
Expectation captures the average or central value of a distribution,
while variance quantifies how much the random variable fluctuates around that center.
Together they constitute the first two \emph{moments} of a distribution.

\begin{definition}[Expectation and Variance]
Let $\RV{X}$ be a \gls{rv} with law $\Dist{X}$ and, when applicable,
probability mass or density function $p(x)$.
The \emph{expectation \gls{expectation}} (or \emph{mean}) of $\RV{X}$ is
\[
\E{X} =
\begin{cases}
\displaystyle \sum_{x} x\,p(x), & \text{if $\RV{X}$ is discrete},\\[0.6em]
\displaystyle \int_{\R} x\,p(x)\,dx, & \text{if $\RV{X}$ is continuous},
\end{cases}
\]
whenever the series or integral converges absolutely.
More generally, for any measurable function $f:\R\to\R$,
\[
\E{f(X)} = \int_{\R} f(x)\,d\Dist{X}(x).
\]

The \emph{variance \gls{variance}} of $\RV{X}$ is
\[
\Var{X} = \E{(X - \E{X})^2},
\]
and its square root $\sigma_X = \sqrt{\Var{X}}$ is the \emph{standard deviation}.
\end{definition}

\begin{remark}
Expectation and variance summarize two complementary aspects of a random variable:
location and dispersion.
While $\E{X}$ describes the typical or central value, 
$\Var{X}$ measures the spread of its probability mass around that value.
They are the first and second moments of the distribution, respectively~\cite{billingsley1995probability,gut2013probability}.
\end{remark}

\begin{definition}[Covariance and Independence]
Given two \glspl{rv} $X$ and $Y$ defined on the same probability space,
their \emph{covariance \gls{covariance}} is
\[
\Cov{X}{Y} = \E{(X - \E{X})(Y - \E{Y})}.
\]
They are said to be \emph{uncorrelated} if $\Cov{X}{Y}=0$.
Moreover, $X$ and $Y$ are \emph{independent} if, for all measurable sets 
$A,B\in\Borel$,
\[
\Prob{X\in A,\,Y\in B} = \Prob{X\in A}\Prob{Y\in B}.
\]
Independence implies uncorrelatedness (provided the variances are finite),
but the converse is not generally true~\cite{billingsley1995probability,gut2013probability}.
\end{definition}

\begin{remark}
Covariance measures the extent to which two random variables vary together:
positive values indicate that they tend to increase or decrease jointly,
negative values indicate opposing tendencies,
and $\Cov{X}{Y}=0$ means the variables are linearly uncorrelated.
Independence, however, is a much stronger property:
it implies that the knowledge of one variable gives no information about the other.
\end{remark}

\begin{remark}[Joint and Marginal Distributions]
When multiple \glspl{rv} are defined on the same probability space,
their combined behavior is described by a \emph{joint distribution}.
For instance, a pair $(X,Y)$ may have joint probability mass or density function $p(x,y)$,
such that for measurable sets $A,B\subseteq\R$,
\[
\Prob{X\in A, Y\in B} = 
\begin{cases}
\displaystyle \sum_{x\in A}\sum_{y\in B} p(x,y), & \text{if discrete},\\[0.6em]
\displaystyle \int_A\!\!\int_B p(x,y)\,dx\,dy, & \text{if continuous}.
\end{cases}
\]
The \emph{marginal distributions} of $X$ and $Y$ are obtained by summing or integrating out the other variable:
\[
p_X(x) = \int_{\R} p(x,y)\,dy, 
\qquad
p_Y(y) = \int_{\R} p(x,y)\,dx.
\]
These operations correspond to the projection of a multivariate measure
onto lower-dimensional spaces~\cite{billingsley1995probability,wasserman2004all}.
\end{remark}

\begin{remark}[Notation]
When a random variable $X$ is fixed or clear from context,
its distribution $\Dist{X}$ is denoted simply by $\ProbMeasure$.
Hence $\E{f(X)} = \int f(x)\,d\ProbMeasure(x)$,
and we often write $\Prob{A}$ for $\Dist{X}(A)$.
\end{remark}

The concepts introduced in this section establish the probabilistic foundation
for statistical analysis.
Having formalized probability spaces, random variables, and their key numerical
characteristics such as expectation and covariance, we now turn to the
\emph{empirical} perspective: how these theoretical quantities can be estimated
from observed data.

\section{Estimation from Data}
\label{sec:estimation}

In most practical settings, the true probability law governing a random variable
is unknown and must be inferred from data.
\emph{Statistical estimation} provides methods to approximate expectations,
variances, and model parameters from a finite sample
drawn from an underlying—often unknown—distribution~\cite{wasserman2004all,casella2002statistical}.

\subsection{Empirical Estimation of Moments}

Let $\{x_i\}_{i=1}^N$ be an independent and identically distributed (i.i.d.) sample 
from a random variable $\RV{X}$ with law $\Dist{X}$.
A natural first step is the empirical estimation of expectation and variance:
\[
\widehat{\E{X}} = \frac{1}{N} \sum_{i=1}^{N} x_i,
\qquad
\widehat{\Var{X+}} = \frac{1}{N} \sum_{i=1}^{N} (x_i - \widehat{\E{X}})^2.
\]
These quantities, known respectively as the \emph{sample mean} and
\emph{sample variance}, approximate their theoretical counterparts $\E{X}$ and $\Var{X}$~\cite{wasserman2004all,casella2002statistical}.

\begin{remark}[Law of Large Numbers]
Let $(X_i)_{i=1}^\infty$ be \gls{iid} \glspl{rv}
with common mean $\E{X_i} = \mu$.
Define the sample mean
\[
    \widehat{\E{X}}_N = \frac{1}{N}\sum_{i=1}^N X_i.
\]

\begin{itemize}
    \item Weak Law of Large Numbers (WLLN).
    If $\E[|X_i|] < \infty$, then
    \[
        \widehat{\E{X}}_N \xrightarrow[]{\mathbb{P}} \mu,
    \]
    meaning that for every $\varepsilon > 0$,
    \[
        \Prob{|\widehat{\E{X}}_N - \mu| > \varepsilon} \;\longrightarrow\; 0
        \quad \text{as } N \to \infty.
    \]
    In words, the probability that the sample mean deviates from the true mean by more than
    $\varepsilon$ becomes negligible as the sample size grows.
    A sufficient condition ensuring this convergence is $\Var{X_i} < \infty$,
    via Chebyshev’s inequality.

    \item Strong Law of Large Numbers (SLLN).
    If $\E{|X_i|} < \infty$ (in particular, if $\E{X_i^2} < \infty$),
    then
    \[
        \widehat{\E{X}}_N \xrightarrow[]{\text{a.s.}} \mu,
    \]
    that is,
    \[
        \Prob{\lim_{N \to \infty} \widehat{\E{X}}_N = \mu} = 1.
    \]
    This stronger result states that for almost every realization (sample path)
    of the sequence $(X_i)$, the sample mean eventually converges to the true mean.
\end{itemize}

Thus, under mild moment assumptions, the sample mean $\widehat{\E{X}}_N$
is a consistent estimator of the population mean $\E{X_i}$:
the Weak Law ensures convergence in probability,
while the Strong Law ensures convergence for almost all sample outcomes~\cite{billingsley1995probability,gut2013probability}.
\end{remark}
\begin{remark}[Central Limit Theorem]
Let $(X_i)_{i=1}^N$ be independent and \gls{iid} \glspl{rv}
with finite mean $\E{X_i} = \mu$ and finite, nonzero variance $\Var{X_i} = \sigma_X^2$.
Define the sample mean
\[
    \widehat{\E{X}}_N = \frac{1}{N}\sum_{i=1}^N X_i.
\]

Then the \emph{Central Limit Theorem (CLT)} states that
\[
    \sqrt{N}\,\frac{\widehat{\E{X}}_N - \mu}{\sigma_X}
    \;\xrightarrow[]{\mathcal{D}}\;
    \mathcal{N}(0,1),
\]
where $\xrightarrow[]{\mathcal{D}}$ denotes \emph{convergence in distribution}.

Specifically, convergence in distribution means that for every bounded and continuous function
$f:\mathbb{R} \to \mathbb{R}$,
\[
    \E{f\big(\sqrt{N}\,(\widehat{\E{X}}_N - \mu)/\sigma_X\big)}
    \;\longrightarrow\;
    \E{f(Z)},
\]
where $Z \sim \mathcal{N}(0,1)$.
Equivalently, the \gls{cdf} of the normalized sample mean
converges pointwise to that of the standard normal distribution:
\[
    \Prob{
        \sqrt{N}\,\frac{\widehat{\E{X}}_N - \mu}{\sigma_X} \leq z
    }
    \;\longrightarrow\;
    \Phi(z),
\]
where $\Phi$ is the \gls{cdf} of $\mathcal{N}(0,1)$.

Consequently, for large $N$, the estimation error of the sample mean is approximately Gaussian
with variance $\sigma_X^2/N$.
This asymptotic normality underlies classical confidence intervals and hypothesis tests~\cite{billingsley1995probability,gut2013probability}.
\end{remark}

\subsection{Parametric Estimation and the Likelihood Principle}

The previous results show how empirical averages converge to their theoretical
counterparts, providing consistent estimators for quantities such as the mean and variance.
However, in many applications one is not only interested in these moments but in
the entire distribution of the data, which may depend on unknown parameters.
This leads to the framework of \emph{parametric estimation}~\cite{casella2002statistical,wasserman2004all}.

Suppose that observations arise from a parametric family of distributions
$\{p(x \mid \theta) : \theta \in \Theta\}$,
where $\theta$ denotes an unknown parameter (possibly multidimensional).
Given an \gls{iid} sample $\{x_i\}_{i=1}^N$, the \emph{likelihood function} is
\[
L(\theta)
  = \prod_{i=1}^N p(x_i \mid \theta),
\]
which evaluates how plausible each parameter value is in light of the data.

It is often more convenient to work with the \emph{log-likelihood},
\[
\ell(\theta)
   = \log L(\theta)
   = \sum_{i=1}^N \log p(x_i \mid \theta).
\]
The log-likelihood enjoys several advantages:
\begin{itemize}
    \item \textbf{Numerical stability.}
    Products of many probabilities can underflow; logarithms avoid this.
    \item \textbf{Additivity.}
    The log-likelihood is a sum rather than a product, simplifying analysis
    and enabling gradient-based optimization.
    \item \textbf{Same maximizer.}
    The logarithm is strictly increasing, so
    \[
        \arg\max_{\theta} L(\theta)
          = \arg\max_{\theta} \ell(\theta).
    \]
    Thus no information is lost by working in log-space~\cite{casella2002statistical}.
\end{itemize}

\paragraph{Maximum likelihood estimation.}
The principle of \emph{maximum likelihood estimation (MLE)} selects
\[
\widehat{\theta}_{\mathrm{MLE}}
    = \arg\max_{\theta \in \Theta}\, L(\theta)
    = \arg\max_{\theta \in \Theta}\, \ell(\theta),
\]
the parameter value that makes the observed sample most probable
under the assumed model~\cite{fisher1925foundations,casella2002statistical}.

\begin{remark}[Asymptotic Properties of MLE]
Under standard regularity conditions, the maximum likelihood estimator
enjoys the following large-sample properties:
\begin{itemize}
    \item \textbf{Consistency:}
    $\widehat{\theta}_{\mathrm{MLE}} \xrightarrow[]{\ProbMeasure} \theta^*$
    as $N \to \infty$.

    \item \textbf{Asymptotic normality:}
    \[
    \sqrt{N}\,\big(\widehat{\theta}_{\mathrm{MLE}} - \theta^*\big)
        \xrightarrow[]{\mathcal{D}}
        \mathcal{N}\!\left(0,\, I(\theta^*)^{-1}\right),
    \]
    where $I(\theta)$ denotes the \emph{Fisher information}.

    \item \textbf{Asymptotic efficiency.}
    The estimator attains the \emph{Cramér--Rao lower bound}:  
    if $\widehat{\theta}$ is any unbiased estimator of $\theta$, then
    \[
        \Var{\widehat{\theta}} \;\ge\; \frac{1}{I(\theta)}.
    \]
    This bound characterizes the minimum achievable variance among unbiased estimators.
\end{itemize}

Together, these properties position MLE as a central tool in both classical
and modern statistical inference~\cite{lehmann1998point,casella2002statistical}.
\end{remark}

\subsection{Example: Gaussian Mean Estimation}

Consider $\RV{X} \sim \mathcal{N}(\mu, \sigma^2)$ with known variance $\sigma^2$
and unknown mean $\mu$.
Given an observed sample $\{x_i\}_{i=1}^N$, the log-likelihood is
\[
\ell(\mu)
= -\frac{N}{2}\log(2\pi\sigma^2)
  - \frac{1}{2\sigma^2} \sum_{i=1}^{N} (x_i - \mu)^2.
\]
Maximizing $\ell(\mu)$ is equivalent to minimizing the sum of squared deviations,
yielding the closed-form estimator
\[
\widehat{\mu}_{\mathrm{MLE}}
    = \frac{1}{N} \sum_{i=1}^{N} x_i.
\]
Thus, the MLE for the Gaussian mean coincides with the sample mean.
This equivalence illustrates the deep connection between likelihood maximization
and least-squares optimization, a principle that recurs throughout statistics
and machine learning~\cite{casella2002statistical,wasserman2004all}.

\begin{remark}
In more general models, parameters may encode location, scale, shape, or mixture weights.
Closed-form solutions are rare, and numerical optimization methods
(e.g., gradient ascent, Newton methods, expectation--maximization)
are routinely employed.
Alternative frameworks—such as Bayesian inference—augment the likelihood
with prior distributions, yielding posterior distributions over parameters.
\end{remark}

\bigskip

In summary, this section connects empirical estimation---supported by
the Law of Large Numbers and Central Limit Theorem---with
model-based inference through the likelihood principle.
These tools form the bridge between probability theory and practical data analysis,
paving the way for more advanced concepts such as information measures,
geometric viewpoints, and kernel-based estimation~\cite{aronszajn1950reproducing,berlinet2004rkhs,scholkopf2002learning}.

\section{Entropy and Information Measures}
\label{sec:entropy}

Statistical moments such as the mean and variance describe specific aspects
of a distribution, but they do not fully capture its overall uncertainty.
Many problems in data analysis, coding, and machine learning require a way to
quantify how unpredictable a random variable is, or how much ``information''
is gained when an outcome is observed.
Entropy and related information-theoretic quantities provide these measures,
linking probability, uncertainty, and information~\cite{shannon1948mathematical,cover2006elements}.

\subsection{Self-Information and Shannon Entropy}

The \emph{self-information} of observing outcome $x$ from a discrete
\gls{rv} $X$ with probability mass $p(x)$ is
\[
I(x) = -\log_b p(x),
\]
where $b$ denotes the logarithm base ($b=2$ for bits and $b=e$ for nats).
Rare events carry high information (large ``surprise''), 
while common events carry little.

\begin{definition}[Shannon Entropy]
For a discrete random variable $X$, the \emph{entropy} is the expected
self-information:
\[
    H(X)
    = \E[I(X)]
    = -\sum_x p(x)\,\log_b p(x).
\]
For a continuous $X$ with density $p(x)$, the analogous quantity is the
\emph{differential entropy}~\cite{shannon1948mathematical,cover2006elements}
\[
    h(X) = -\int_{\R} p(x)\,\log_b p(x)\,dx.
\]
\end{definition}

\begin{remark}
Entropy measures the average uncertainty of a random variable.
For example, a uniform distribution on $M$ equally likely outcomes
achieves the maximum entropy $H(X)=\log_b M$,
while a degenerate distribution with all mass on a single point has entropy $0$.
\end{remark}

\subsection{Joint, Conditional, and Mutual Information}

Entropy naturally extends to multiple random variables.
Let $X$ and $Y$ be discrete with joint pmf $p(x,y)$.
\begin{itemize}
    \item \emph{Joint entropy}:
    \[
        H(X,Y) = -\sum_{x,y} p(x,y)\log_b p(x,y)
    \]
    measures the total uncertainty of the pair~\cite{cover2006elements}.
    \item \emph{Conditional entropy}:
    \[
        H(Y\mid X) = -\sum_{x,y} p(x,y)\log_b p(y\mid x)
    \]
    measures the remaining uncertainty in $Y$ once $X$ is known.
    \item \emph{Mutual information}:
    \[
        I(X;Y)
        = H(X) + H(Y) - H(X,Y)
    \]
    quantifies the reduction in uncertainty in $X$ due to observing $Y$
    (and vice versa).
\end{itemize}

\begin{remark}
Mutual information is always nonnegative and equals zero if and only if
$X$ and $Y$ are independent.
Thus it provides a general measure of statistical dependence,
capturing nonlinear relationships that variance-based measures (like covariance)
may miss~\cite{cover2006elements}.
\end{remark}

\subsection{Relative Entropy and Divergences}

To compare two probability distributions $P$ and $Q$ on the same domain,
information theory uses the \emph{Kullback--Leibler (KL) divergence}:
\[
D_{\mathrm{KL}}(P\|Q)
    = \sum_x P(x)\,\log_b \frac{P(x)}{Q(x)}.
\]
The continuous analogue replaces the sum by an integral.
Although asymmetric, $D_{\mathrm{KL}}$ satisfies
$D_{\mathrm{KL}}(P\|Q)\ge 0$,
with equality if and only if $P=Q$ almost everywhere~\cite{kullback1951information,cover2006elements}.

\begin{remark}
KL divergence measures how hard it is to represent data drawn from
$P$ using a coding scheme tailored to $Q$.
It appears in maximum likelihood estimation, Bayesian inference,
variational methods, and many optimization problems in machine learning.
\end{remark}

\subsection{Kolmogorov--Nagumo Means and Generalized Entropies}

Expectation can be viewed as a special case of a more general framework of
averaging, introduced by Kolmogorov and Nagumo.
Given a strictly monotonic and continuous function $g$,
the \emph{Kolmogorov--Nagumo (KN) mean} of a discrete random variable $A$ with
pmf $p(a)$ is
\[
    \E[g]{A}
    = g^{-1}\!\left(\sum_a p(a)\,g(a)\right).
\]
The choice $g(a)=a$ yields the usual arithmetic mean.
Other choices of $g$ control how the averaging operator weights extreme,
rare, or dominant outcomes~\cite{kolmogorov1930moyenne,renyi1961measures,jizba2004generalized}.

These generalized means give rise to alternative entropy definitions.
Different transformations $g$ lead to generalized notions
of uncertainty, two important examples being Rényi and Tsallis entropies.

\subsection{Rényi and Tsallis Entropies}

For $\alpha>0$, $\alpha\neq 1$, the \emph{Rényi entropy} of order $\alpha$ is
\[
    H_\alpha(X)
    = \frac{1}{1-\alpha}
      \log_b \sum_x p(x)^{\alpha}.
\]
This corresponds to choosing $g(a)=a^{1-\alpha}$ in the KN-mean framework.
The parameter $\alpha$ adjusts sensitivity: $\alpha<1$ emphasizes rare events,
while $\alpha>1$ emphasizes likely outcomes~\cite{renyi1961measures,jizba2004generalized}.

\begin{proposition}[Convergence to Shannon Entropy]
The Rényi entropy converges to Shannon entropy as $\alpha \to 1$~\cite{renyi1961measures}:
\[
    \lim_{\alpha\to 1} H_\alpha(X)
    = -\sum_x p(x)\log_b p(x)
    = H(X).
\]
\end{proposition}

\begin{proof}
The proof follows by applying L'Hôpital's rule to the defining expression.
\[
\begin{aligned}
\lim_{\alpha \to 1} H_\alpha(X)
&= \lim_{\alpha \to 1} 
   \frac{1}{1-\alpha}\log_b\!\left(\sum_x p(x)^\alpha\right)  \\[0.25em]
&= -\frac{1}{\ln(b)}
   \lim_{\alpha\to 1}
   \frac{\sum_x p(x)^\alpha \ln p(x)}{\sum_x p(x)^\alpha}  \\[0.25em]
&= -\sum_x p(x)\log_b p(x).
\end{aligned}
\]
\end{proof}

A closely related quantity is the \emph{Tsallis entropy},
\[
    S_q(X) = \frac{1 - \sum_x p(x)^q}{q-1},
\]
which also reduces to Shannon entropy as $q\to 1$.
Where Rényi entropy keeps a logarithmic averaging structure,
Tsallis entropy uses a linearized form, often applied in
statistical physics and robust statistical modeling~\cite{tsallis1988possible,jizba2004generalized}.

Entropy and its generalizations provide a unified language for expressing
uncertainty, diversity, and information.
They appear in areas ranging from coding theory and thermodynamics
to modern machine learning, where they motivate loss functions,
regularization strategies, and information-based criteria.
In later chapters, these concepts will connect to geometric and functional
perspectives, offering deeper insights into how similarity and uncertainty
can be quantified in high-dimensional or continuous spaces.

\chapter{Hilbert Spaces and the Geometry of Random Variables}
\label{chap:hilbert_spaces}

Random variables can be studied not only through algebra and calculus,
but also through geometry.
By viewing random variables as elements of an inner-product space,
statistical concepts such as expectation, variance, and regression
acquire geometric interpretations.
This chapter introduces Hilbert spaces as the mathematical setting in which
these ideas naturally arise~\cite{parthasarathy2005introduction,ash1972probability,scholkopf2002learning}.

\section{Hilbert Spaces: Definition and Geometry}
\label{sec:hilbert_definitions}

Hilbert spaces generalize Euclidean geometry to possibly infinite dimensions.
They provide a unifying framework for geometry, functional analysis, and
statistics, in which notions such as distance, angle, orthogonality, and
projection are defined through an inner product~\cite{debnath2005hilbert,reed1980functional,conway2007course}.

\begin{definition}[Metric Space]
A \emph{metric space} is a pair $(M,d)$ consisting of a set $M$ and a function
$d : M \times M \to [0,\infty)$, called a \emph{metric}, satisfying for all
$x,y,z \in M$:
\begin{enumerate}
    \item \textbf{Positivity:} $d(x,y) \ge 0$ and $d(x,y)=0$ if and only if $x=y$;
    \item \textbf{Symmetry:} $d(x,y) = d(y,x)$;
    \item \textbf{Triangle inequality:} 
          $d(x,z) \le d(x,y) + d(y,z)$.
\end{enumerate}
The metric $d(x,y)$ specifies the distance between $x$ and $y$.
\end{definition}

\begin{definition}[Normed Space]
Let $V$ be a vector space over $\mathbb{K}\in\{\R,\C\}$.
A \emph{norm} is a function $\|\cdot\| : V \to [0,\infty)$ satisfying, for all
$x,y\in V$ and all scalars $\alpha\in\mathbb{K}$:
\begin{enumerate}
    \item \textbf{Positivity:} $\|x\| \ge 0$ and $\|x\|=0$ if and only if $x=0$;
    \item \textbf{Homogeneity:} $\|\alpha x\| = |\alpha|\,\|x\|$;
    \item \textbf{Triangle inequality:} $\|x+y\| \le \|x\| + \|y\|$.
\end{enumerate}
A pair $(V,\|\cdot\|)$ is called a \emph{normed space}.  
Every norm induces a metric via $d(x,y)=\|x-y\|$.
\end{definition}

\begin{definition}[Inner Product Space]
Let $\mathcal{H}$ be a vector space over $\mathbb{K} \in \{\R,\C\}$.
An \emph{inner product} on $\mathcal{H}$ is a mapping
\[
    \langle \cdot , \cdot \rangle : \mathcal{H} \times \mathcal{H} \to \mathbb{K}
\]
satisfying, for all $f,g,h \in \mathcal{H}$ and scalars $a,b \in \mathbb{K}$:
\begin{enumerate}
    \item \textbf{Linearity in the first argument:}
          $\langle a f + b g , h \rangle
           = a\langle f,h\rangle + b\langle g,h\rangle$;
    \item \textbf{Conjugate symmetry:}
          $\langle f,g\rangle = \overline{\langle g,f\rangle}$;
    \item \textbf{Positive definiteness:}
          $\langle f,f\rangle \ge 0$, with equality if and only if $f=0$.
\end{enumerate}
\end{definition}

An inner product induces a norm
\[
    \|f\| = \sqrt{\langle f,f\rangle},
\]
and therefore a metric $d(f,g)=\|f-g\|$.  
Thus every inner-product space is a normed space and, consequently, a metric space.

\begin{definition}[Cauchy Sequence]
A sequence $(f_n)_{n\ge 1}$ in a normed space $(V,\|\cdot\|)$ is called a
\emph{Cauchy sequence} if, for every $\varepsilon > 0$, there exists
$N\in\mathbb{N}$ such that
\[
    \| f_n - f_m \| < \varepsilon
    \qquad\text{for all } n,m \ge N.
\]
In other words, the elements of the sequence eventually become arbitrarily
close to each other with respect to the norm.
\end{definition}

\begin{definition}[Hilbert Space]
A \emph{Hilbert space} is a complete inner-product space
$(\mathcal{H},\langle\cdot,\cdot\rangle)$,
meaning that every Cauchy sequence in $\mathcal{H}$ converges to a limit
that also belongs to $\mathcal{H}$~\cite{debnath2005hilbert,reed1980functional,conway2007course}.
\end{definition}

Completeness ensures that limit operations familiar from $\R^n$ extend to
infinite-dimensional settings.  
Since an inner product induces both a norm and a metric,
Hilbert spaces sit at the top of the following hierarchy:
\[
\text{Hilbert space}
\;\Longrightarrow\;
\text{inner-product space}
\;\Longrightarrow\;
\text{normed space}
\;\Longrightarrow\;
\text{metric space}.
\]

\begin{example}[Classical Hilbert Spaces]
We illustrate two fundamental Hilbert spaces that arise throughout analysis and
probability.

\begin{itemize}
    \item \textbf{The space $\ell^2$ of square-summable sequences.}

    This space consists of all infinite sequences of real or complex numbers
    \[
        x = (x_1,x_2,\dots)
        \qquad\text{such that}\qquad
        \sum_{n=1}^\infty |x_n|^2 < \infty.
    \]
    The inner product
    \[
        \langle x , y\rangle
        = \sum_{n=1}^\infty x_n \,\overline{y_n}
    \]
    generalizes the Euclidean dot product to infinitely many coordinates.
    The condition $\sum |x_n|^2 < \infty$ guarantees that the series defining the
    inner product converges.

    Intuitively, $\ell^2$ contains sequences whose ``energy'' or ``length''
    remains finite in the limiting sense.
    Completeness here means that any Cauchy sequence of such vectors converges
    to another square-summable sequence.

    \item \textbf{The space $L^2(\Omega,\mathcal{F},\mu)$ of square-integrable functions.}

    This space consists of all measurable functions
    \[
        f : \Omega \to \mathbb{K}
        \qquad\text{such that}\qquad
        \int_\Omega |f|^2\, d\mu < \infty.
    \]
    Two functions are considered the same element of $L^2$ if they differ only
    on a set of measure zero.  
    The inner product is
    \[
        \langle f , g \rangle
        = \int_\Omega f(\omega)\,\overline{g(\omega)}\, d\mu(\omega).
    \]

    This space generalizes the idea of the dot product to functions:
    instead of summing coordinate products, we integrate pointwise products.
    The requirement $\int |f|^2 < \infty$ ensures finite “energy’’ or
    “signal strength.’’

    A key fact is that $L^2$ is complete: limits of mean-square convergent
    sequences of functions always remain in the space.  
    This makes $L^2$ the natural home for random variables with
    finite variance~\cite{debnath2005hilbert,reed1980functional,parthasarathy2005introduction}.
\end{itemize}
\end{example}

\begin{remark}[Geometric Interpretation]
Inner-product spaces extend Euclidean geometry:
\begin{itemize}
    \item \textbf{Orthogonality:}
          $f \perp g$ if $\langle f,g\rangle = 0$, meaning the vectors
          ``point in independent directions.''
    \item \textbf{Pythagorean identity:}
          if $f \perp g$, then
          \[
              \|f+g\|^2 = \|f\|^2 + \|g\|^2.
          \]
          This generalizes the familiar right-triangle relation in $\R^2$ and $\R^3$.
    \item \textbf{Projection theorem:}
          for any closed subspace $\mathcal{M}$ and any $f\in\mathcal{H}$,
          there exists a unique $f_\mathcal{M}\in\mathcal{M}$ minimizing
          $\|f-g\|$ over $g\in\mathcal{M}$.
          Moreover, the error $f - f_\mathcal{M}$ is orthogonal to $\mathcal{M}$.
          This is the geometric principle underlying least-squares approximation~\cite{conway2007course,debnath2005hilbert}.
\end{itemize}

These ideas are not merely formal: in the next section we apply them to
$L^2(\Omega)$, where random variables behave exactly like vectors and
statistical quantities acquire clear geometric meaning.
\end{remark}

\section{The Hilbert Space \texorpdfstring{$L^2$}{L2} and Probabilistic Geometry}
\label{sec:L2_geometry}

\subsection{Definition of Measure Spaces}

\begin{definition}[Measure Space]
Let $X$ be a set and let $\mathcal{A}$ be a $\sigma$-algebra of subsets of $X$.

A \emph{measure} on $(X,\mathcal{A})$ is a function
\[
    \mu : \mathcal{A} \to [0,\infty]
\]
that assigns a nonnegative ``size'' or ``weight'' to each measurable set and
satisfies the property of \emph{countable additivity}:
whenever $(A_i)_{i=1}^\infty$ are pairwise disjoint sets in $\mathcal{A}$,
\[
    \mu\!\left(\bigcup_{i=1}^\infty A_i\right)
    = \sum_{i=1}^\infty \mu(A_i).
\]
The value $\mu(\emptyset)=0$ follows automatically from this axiom.

A triplet $(X,\mathcal{A},\mu)$ consisting of a set, a $\sigma$-algebra, and a
measure is called a \emph{measure space}~\cite{billingsley1995probability,ash1972probability}.
\end{definition}

To describe random variables using geometric ideas, we need a space in which
functions can be compared, measured, and combined much like vectors in
$\R^n$.
$L^p$ spaces provide this framework.
We begin with the general idea, and then specialize to probability~\cite{reed1980functional,conway2007course,bogachev2007measure}.

\subsection{Definition of \texorpdfstring{$L^p$}{Lp} Spaces}

Let $(X,\mathcal{A},\mu)$ be a measure space and let $1 \le p < \infty$.
A measurable function $f : X \to \mathbb{K}$ (with $\mathbb{K}\in\{\R,\C\}$) is said
to be \emph{$p$-integrable} if
\[
    \int_X |f(x)|^p\, d\mu(x) < \infty.
\]
Two measurable functions that differ only on a set of measure zero are identified,
since they are indistinguishable from the point of view of integration.

\begin{definition}[$L^p$ Space]
For $1 \le p < \infty$, the space $L^p(X,\mathcal{A},\mu)$ is defined as
\[
    L^p(X,\mathcal{A},\mu)
    = \big\{ f : X \to \mathbb{K} \text{ measurable } \big|\,
       \int_X |f|^p\, d\mu < \infty \big\},
\]
where functions equal $\mu$-almost everywhere are regarded as the same element.
\end{definition}

\begin{proposition}[$L^p$ is a Normed Space]
For $1 \le p < \infty$, the quantity
\[
    \|f\|_p := \left( \int_X |f(x)|^p\, d\mu(x) \right)^{1/p}
\]
defines a norm on $L^p(X,\mu)$.
Thus $(L^p(X,\mu),\|\cdot\|_p)$ is a normed space.
\end{proposition}

\begin{proof}[Idea of the proof]
The expression $\|f\|_p$ behaves like a length:
\begin{itemize}
    \item $\|f\|_p\ge 0$, and $\|f\|_p=0$ only if $f$ is zero almost everywhere;
    \item scaling by a constant scales the norm: $\|\alpha f\|_p=|\alpha|\|f\|_p$;
    \item the triangle inequality $\|f+g\|_p \le \|f\|_p+\|g\|_p$ follows from
          classical inequalities due to Hölder and Minkowski.
\end{itemize}
Thus $L^p$ functions behave very much like vectors with “length’’ $\|f\|_p$~\cite{reed1980functional,conway2007course}.
\end{proof}

A deep theorem in analysis states that every $L^p$ space ($1\le p<\infty$) is
\emph{complete} with respect to this norm, meaning that Cauchy sequences of
functions have limits in the same space.
Such spaces are called \emph{Banach spaces}~\cite{reed1980functional,conway2007course}.

\begin{proposition}[$L^2$ as a Hilbert Space]
On $L^2(X,\mu)$, the quantity
\[
    \langle f,g\rangle := \int_X f(x)\,\overline{g(x)}\, d\mu(x)
\]
defines an inner product whose induced norm satisfies $\|f\|_2=\sqrt{\langle f,f\rangle}$.
Since $L^2(X,\mu)$ is also complete under this norm, it is a Hilbert space~\cite{reed1980functional,conway2007course,parthasarathy2005introduction}.
\end{proposition}

\begin{remark}[Why Only $L^2$ is a Hilbert Space]
A norm comes from an inner product if and only if it satisfies the
\emph{parallelogram law}:
\[
    \|u+v\|^2 + \|u-v\|^2 = 2\|u\|^2 + 2\|v\|^2.
\]
The $L^2$ norm satisfies this identity, allowing one to reconstruct the inner
product from $\|\cdot\|_2$.
For $p\neq 2$, the norm $\|\cdot\|_p$ does \emph{not} satisfy this law (even in
simple spaces like $\R^2$), so it cannot arise from an inner product.
Thus $L^p$ is a Hilbert space if and only if $p=2$~\cite{reed1980functional,conway2007course,debnath2005hilbert}.
\end{remark}

\subsubsection*{Specialization to Random Variables}

In probability theory, the underlying measure space is a probability space
$(\Omega,\mathcal{F},P)$.
In this case,
\[
    L^p(\Omega) := L^p(\Omega,\mathcal{F},P)
    = \{ X : \Omega\to\R \,\mid\, \E{|X|^p} < \infty \},
\]
so $L^p(\Omega)$ consists precisely of $p$-integrable random variables.
For $p=2$, the space $L^2(\Omega)$ is a Hilbert space under the inner product
\[
    \langle X,Y\rangle = \E{XY}.
\]
This makes it an ideal setting for a geometric interpretation of expectation,
variance, covariance, and regression~\cite{parthasarathy2005introduction,ash1972probability}.

\subsection{Expectation, Variance, and Covariance in \texorpdfstring{$L^2$}{L2}}

In $L^2(\Omega)$, random variables behave like vectors.
Statistical quantities correspond to their geometric relationships~\cite{parthasarathy2005introduction,ash1972probability}:
\begin{itemize}
    \item \textbf{Expectation:} 
          $\E{X}$ is the projection of $X$ onto the space of constant functions.
    \item \textbf{Variance:} 
          $\Var{X}=\|X - \E{X}\|_2^2$ is the squared distance from $X$ to its mean.
    \item \textbf{Covariance:} 
          $\Cov{X}{Y}=\langle X-\E{X},\,Y-\E{Y}\rangle$.
\end{itemize}

\begin{remark}
$X$ and $Y$ are uncorrelated precisely when they are orthogonal in $L^2$.
Independence always implies orthogonality, but the converse need not hold~\cite{parthasarathy2005introduction,ash1972probability}.
\end{remark}

\subsection{Linear Regression as a Projection Problem}

Let $\mathcal{H}$ be the subspace spanned by $\{1,X_1,\dots,X_d\}$.
Given a response variable $Y$, the best linear predictor of $Y$
in the mean-square sense is the orthogonal projection of $Y$ onto $\mathcal{H}$:
\[
\widehat{Y}
    = \E{Y} + \sum_{j=1}^d \beta_j (X_j - \E{X_j}),
\]
where the coefficients $\beta_j$ minimize $\E{(Y-\widehat{Y})^2}$.
The projection theorem guarantees existence and uniqueness~\cite{aitken1935least,frisch1934statistical,searle1971linear,parthasarathy2005introduction}.

\begin{remark}
This shows that ordinary least squares (OLS) is fundamentally
a geometric projection problem in $L^2$,
not merely an algebraic computation~\cite{searle1971linear,seber2003linear}.
\end{remark}

\subsection{Higher-Order Moments as Tensors}

Beyond second-order structure, higher moments describe features such as skewness
and kurtosis.
For $k \ge 3$, define
\[
\mu_k(X) = \E{(X-\E{X})^k},
\qquad
T_k(X_1,\ldots,X_k)
    = \E{(X_1-\E{X_1})\cdots(X_k-\E{X_k})}.
\]
These objects can be viewed as symmetric multilinear forms,
or equivalently as tensors, encoding multi-way interactions among variables~\cite{mardia1970measures,amaratunga2001analysis}.

\begin{table}[h!]
\centering
\renewcommand{\arraystretch}{1.2}
\begin{tabular}{@{}lll@{}}
\toprule
\textbf{Moment Type} & \textbf{Definition} & \textbf{Interpretation in $L^2$} \\
\midrule
Expectation & $\E{X}$ & Projection onto constants \\
Variance & $\E{(X-\E{X})^2}$ & Squared distance to the mean \\
Covariance & $\E{(X-\E{X})(Y-\E{Y})}$ & Inner product of centered variables \\
Standardized variable & $\tilde X = \frac{X-\E{X}}{\sigma_X}$ & Lies on the $L^2$ unit sphere \\
Higher-order moment & $\E{\tilde X^k}$ & Shape information (skewness, kurtosis) \\
Moment tensor & $\E{X^{\otimes k}}$ & $k$-way dependence structure \\
\bottomrule
\end{tabular}
\caption{Moment-based quantities and their geometric interpretation in $L^2$.~\cite{mardia1970measures,amaratunga2001analysis}}
\end{table}

\medskip
\noindent
In summary, the $L^2(\Omega)$ perspective reveals that many statistical notions
have natural geometric interpretations.
This viewpoint will serve as the foundation for the next chapter, where inner
products are generalized through \emph{kernels} to extend these ideas to
nonlinear and high-dimensional settings~\cite{aronszajn1950reproducing,berlinet2004rkhs,scholkopf2002learning}.

\chapter{Kernel Methods and Reproducing Kernel Hilbert Spaces}
\label{chap:kernels}

Kernel methods generalize linear algebra and Euclidean geometry to nonlinear feature spaces.
They allow statistical and learning algorithms to operate in high- or even infinite-dimensional
spaces without explicitly computing feature coordinates~\cite{scholkopf2002learning,cristianini2000introduction,hofmann2008kernel}
.
This chapter formalizes the concepts of positive definite kernels,
\gls{rkhs}, and Hilbert--Schmidt operators,
and connects them to probabilistic notions such as covariance and entropy~\cite{aronszajn1950reproducing,berlinet2004reproducing,muandet2017kernel}
.

\section{Positive Definite Kernels}
\label{sec:positive_definite_kernels}

\subsection*{Definition and Properties}

Let $\mathcal{X}$ be a non-empty set.
A \emph{positive definite kernel} is a symmetric function
$k : \mathcal{X}\times\mathcal{X}\to\R$ satisfying
\[
\sum_{i=1}^{N}\sum_{j=1}^{N} c_i c_j\, k(x_i,x_j) \ge 0
\quad
\text{for all } N\in\N, \; x_1,\dots,x_N\in\mathcal{X}, \;
c_1,\dots,c_N\in\R.
\]
For any finite collection $\{x_i\}_{i=1}^N$, the matrix
$K=[k(x_i,x_j)]_{i,j=1}^N$ is symmetric and positive semidefinite,
and is called a \emph{Gram matrix}~\cite{aronszajn1950reproducing,berlinet2004reproducing}
.
This property ensures that $k$ behaves like an inner product in some (possibly
high-dimensional) feature space.

\begin{example}[Classical Kernels]
Common examples include:
\begin{itemize}
    \item \textbf{Linear:} $k(x,x') = \langle x, x'\rangle$,
    \item \textbf{Polynomial:} $k(x,x') = (\langle x, x'\rangle + c)^p$, with $c \ge 0$, $p\in\N$,
    \item \textbf{Gaussian (RBF):} $k(x,x') = \exp(-\|x-x'\|^2 / 2\sigma^2)$,
    \item \textbf{Laplacian:} $k(x,x') = \exp(-\|x-x'\|_1 / \sigma)$.
\end{itemize}
These kernels induce different geometries in feature space and correspond to
different notions of similarity on $\mathcal{X}$~\cite{scholkopf2002learning,cristianini2000introduction}
.
\end{example}

\subsection*{Feature Space Interpretation}

Every positive definite kernel defines an implicit mapping
\[
\Phi: \mathcal{X} \to \mathcal{H}
\]
into a Hilbert space $\mathcal{H}$ such that
\[
k(x,x') = \langle \Phi(x), \Phi(x') \rangle_{\mathcal{H}}.
\]

The vectors $\Phi(x)$ are called \emph{feature representations} of the inputs.
Within $\mathcal{H}$, the kernel $k$ plays the role of an inner product,
inducing distances, angles, and projections among feature vectors~\cite{aronszajn1950reproducing,berlinet2004reproducing}.

\begin{remark}
The \emph{kernel trick} exploits this representation:
whenever an algorithm can be expressed solely in terms of inner products
$\langle \Phi(x),\Phi(x')\rangle_{\mathcal{H}}$, one can replace these by
$k(x,x')$ and work directly with the kernel.
This allows nonlinear regression, classification, or principal component
analysis to be carried out without ever constructing the feature map
$\Phi$ explicitly~\cite{scholkopf2002learning,hofmann2008kernel}
.
\end{remark}

\begin{figure}[t]
    \centering
    \includegraphics[width=0.75\linewidth]{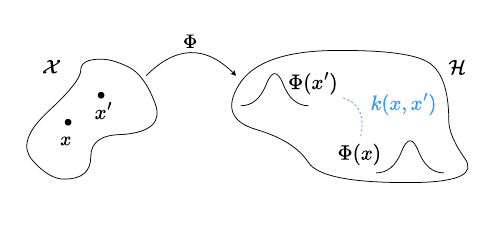}
    \caption{
    Illustration of the feature mapping $\Phi$ induced by a Gaussian kernel.
    Each input $x \in \mathcal{X}$ is mapped to a feature vector $\Phi(x)$ in
    the Hilbert space $\mathcal{H}$, and the kernel values satisfy
    $k(x,x') = \langle \Phi(x), \Phi(x') \rangle_{\mathcal{H}}$.
    }
    \label{fig:feature_mapping}
\end{figure}

\section{Reproducing Kernel Hilbert Spaces (RKHS)}
\label{sec:rkhs}

Reproducing Kernel Hilbert Spaces provide a rigorous framework for the feature
spaces induced by positive definite kernels.

\begin{definition}[Reproducing Kernel Hilbert Space]
Let $\mathcal{X}$ be a non-empty set and $\mathcal{H}$ a Hilbert space of functions 
$f:\mathcal{X}\to\R$ with inner product $\langle \cdot, \cdot \rangle_{\mathcal{H}}$.
We say $\mathcal{H}$ is a \emph{Reproducing Kernel Hilbert Space (RKHS)} 
if there exists a function $k:\mathcal{X}\times\mathcal{X}\to\R$ such that~\cite{aronszajn1950reproducing,berlinet2004reproducing}
:
\begin{enumerate}
    \item For every $x\in\mathcal{X}$, the function $k(x,\cdot)$ belongs to $\mathcal{H}$;
    \item \emph{Reproducing property:} 
    for all $f\in\mathcal{H}$ and all $x\in\mathcal{X}$,
    \[
        \langle f, k(x,\cdot)\rangle_{\mathcal{H}} = f(x).
    \]
\end{enumerate}
\end{definition}

The function $k$ is called the \emph{reproducing kernel} of $\mathcal{H}$.
The reproducing property shows that pointwise evaluation is a continuous linear
functional on $\mathcal{H}$, a fact guaranteed by the Riesz Representation
Theorem~\cite{berlinet2004reproducing}
.

\begin{theorem}[Riesz Representation Theorem]
\label{thm:riesz}
Let $\mathcal{H}$ be a Hilbert space over $\R$, 
and let $\varphi : \mathcal{H} \to \R$ be a bounded (i.e., continuous) linear functional.
Then there exists a unique element $h_\varphi \in \mathcal{H}$ such that~\cite{conway2007course,reed1980functional,debnath2005hilbert}

\[
\varphi(f) = \langle f, h_\varphi \rangle_{\mathcal{H}} 
\quad \text{for all } f \in \mathcal{H}.
\]
\end{theorem}

\begin{remark}
Riesz representation states that every continuous linear functional on a Hilbert
space can be expressed as an inner product with a unique element of that space.
Applied to the evaluation functional $\delta_x(f) = f(x)$, it guarantees the
existence of a unique element $k(x,\cdot) \in \mathcal{H}$ such that
\[
f(x) = \langle f, k(x,\cdot) \rangle_{\mathcal{H}},
\]
which is precisely the reproducing property.
\end{remark}

\begin{theorem}[Moore--Aronszajn]
For every positive definite kernel $k$ on $\mathcal{X}$, 
there exists a unique Hilbert space $\mathcal{H}_k$ of functions $f:\mathcal{X}\to\R$
such that $k$ is its reproducing kernel and~\cite{aronszajn1950reproducing}

\[
\langle k(x,\cdot), k(x',\cdot)\rangle_{\mathcal{H}_k} = k(x,x')
\quad\text{for all } x,x'\in\mathcal{X}.
\]
Conversely, every RKHS has a unique reproducing kernel of this form.
\end{theorem}

\begin{example}[Gaussian RKHS]
For the Gaussian kernel
\[
k(x,x') = \exp\!\left(-\frac{\|x-x'\|^2}{2\sigma^2}\right),
\]
the associated RKHS $\mathcal{H}_k$ consists of functions that can be expressed
as limits (in the RKHS norm) of finite linear combinations of Gaussian bumps:
\[
f(\cdot) = \sum_{i=1}^N \alpha_i\, k(x_i,\cdot).
\]
For two such functions
\(
f(\cdot) = \sum_{i=1}^N \alpha_i\, k(x_i,\cdot)
\)
and
\(
g(\cdot) = \sum_{j=1}^M \beta_j\, k(x'_j,\cdot),
\)
the inner product is
\[
\langle f,g\rangle_{\mathcal{H}_k}
= \sum_{i=1}^N\sum_{j=1}^M \alpha_i\beta_j\, k(x_i,x'_j).
\]~\cite{berlinet2004reproducing,scholkopf2002learning}

\end{example}

\begin{remark}
An RKHS extends linear geometry to spaces of functions:
inner products between functions become kernel evaluations between data points.
This viewpoint allows Hilbert-space tools---such as projections, angles, and
eigen-decompositions---to be applied directly in statistical learning problems.
\end{remark}

\section{Hilbert--Schmidt and Covariance Operators}
\label{sec:hs_operators}

\subsection*{Definition and Basic Properties}

Let $\mathcal{H}_1$ and $\mathcal{H}_2$ be Hilbert spaces.
A bounded linear operator $T:\mathcal{H}_1\to\mathcal{H}_2$ 
is called a \emph{Hilbert--Schmidt operator} if there exists 
an orthonormal basis $\{e_i\}$ of $\mathcal{H}_1$ such that
\[
\|T\|^2_{\mathrm{HS}}
= \sum_{i=1}^{\infty} \|T e_i\|_{\mathcal{H}_2}^2 < \infty.
\]
The quantity $\|T\|_{\mathrm{HS}}$ is called the \emph{Hilbert--Schmidt norm}~\cite{reed1980functional,debnath2005hilbert}
.

\begin{example}[Integral Operators]
Let $(\mathcal{X},\mu)$ be a measure space and $k:\mathcal{X}\times\mathcal{X}\to\R$
a square-integrable kernel.
The operator $T_k : L^2(\mathcal{X}) \to L^2(\mathcal{X})$ defined by
\[
(T_k f)(x) = \int_{\mathcal{X}} k(x,x') f(x')\, d\mu(x')
\]
is Hilbert--Schmidt, with
\[
\|T_k\|_{\mathrm{HS}}^2
= \int_{\mathcal{X}^2} |k(x,x')|^2\, d\mu(x)\,d\mu(x').
\]~\cite{reed1980functional,debnath2005hilbert}

\end{example}

\begin{remark}
Hilbert--Schmidt operators generalize finite-dimensional matrices
to infinite-dimensional settings: they are compact operators
whose entries have finite total squared magnitude, analogous to matrices with
finite Frobenius norm.
\end{remark}

\subsection*{Covariance Operators in RKHS}

Let $\mathcal{H}_k$ be an RKHS associated with a kernel $k$, and 
let $X$ be a random variable taking values in $\mathcal{X}$.
The \emph{mean embedding} of the distribution $P_X$ into $\mathcal{H}_k$ is
\[
\mu_X = \E[X]{k(X,\cdot)} \in \mathcal{H}_k.
\]~\cite{smola2007hilbert,sriperumbudur2010hilbert,muandet2017kernel}

The \emph{covariance operator} of $X$ in $\mathcal{H}_k$ is defined as
\[
C_X
= \E[X]{(k(X,\cdot) - \mu_X) \otimes (k(X,\cdot) - \mu_X)},
\]
which acts on functions $f \in \mathcal{H}_k$ via
\[
C_X f
= \E[X]{(f(X) - \E[f(X)])\, k(X,\cdot)}.
\]~\cite{smola2007hilbert,gretton2005hsic,muandet2017kernel}

\begin{remark}
The operator $C_X$ generalizes the covariance matrix to nonlinear feature spaces.
It is positive, self-adjoint, and trace-class on $\mathcal{H}_k$, and forms the
basis for kernel PCA, kernel CCA, and several dependence measures~\cite{gretton2005hsic,gretton2012kernel}
.
\end{remark}

\section{Information Measures in \glspl{rkhs}}
\label{sec:information_rkhs}

Entropy and divergence measures introduced earlier can be estimated
and interpreted geometrically in an \gls{rkhs} framework.

\subsection*{Kernel Density Estimation (KDE)}

Given samples $\{x_n\}_{n=1}^N \subset \mathcal{X}$,
the \emph{Parzen window} or \emph{kernel density estimator} is
\[
\widehat{p}(x) = \frac{1}{N}\sum_{n=1}^N k_\sigma(x, x_n),
\]
where $k_\sigma$ is a normalized positive kernel 
(e.g., a Gaussian with bandwidth $\sigma$).
This nonparametric estimate provides a flexible way to approximate the
underlying density from finite data and forms the starting point for
kernel-based entropy estimators~\cite{rosenblatt1956remarks,parzen1962estimation,silverman1986density}
.

\subsection*{Estimating Rényi Entropy with Kernels}

For Rényi entropy of order $\alpha=2$,
the empirical estimator based on KDE can be written as
\[
\widehat{H}_2(X)
= -\log \sum_{n=1}^N \widehat{p}(x_n)^2
= -\log \!\left(
\frac{1}{N^2}\sum_{n=1}^{N}\sum_{m=1}^{N} k_{\sqrt{2}\sigma}(x_n,x_m)
\right).
\]
This expression admits a natural geometric interpretation in $\mathcal{H}_k$~\cite{principe2010itl}
.

\begin{remark}
Let 
\[
\mu_{\widehat{P}}
= \frac{1}{N}\sum_{n=1}^N k_{\sqrt{2}\sigma}(x_n,\cdot)
\]
be the empirical mean embedding of the sample in the \gls{rkhs}.
Then
\[
\widehat{H}_2(X)
= -\log \|\mu_{\widehat{P}}\|_{\mathcal{H}_k}^2.
\]
Thus, the Rényi entropy of order~2 is (up to a logarithm) the squared norm of
the embedded empirical distribution:
entropy measures how spread out the sample is in feature space.
Highly concentrated embeddings correspond to low entropy; diffuse embeddings
correspond to high entropy~\cite{smola2007hilbert,muandet2017kernel}
.
\end{remark}

\subsection*{Mean Embeddings and Maximum Mean Discrepancy (MMD)}

The \gls{rkhs} mean embedding $\mu_X = \E{k(X,\cdot)}$ represents a distribution as
an element of $\mathcal{H}_k$.
Given two random variables $X$ and $Y$, their discrepancy can be quantified by the 
\emph{Maximum Mean Discrepancy} (MMD):
\[
\mathrm{MMD}^2(X,Y)
= \|\mu_X - \mu_Y\|_{\mathcal{H}_k}^2
= \E[k(X,X')] + \E[k(Y,Y')] - 2\E[k(X,Y)],
\]
where expectations are taken over independent copies
$X,X'\sim P_X$ and $Y,Y'\sim P_Y$.~\cite{gretton2012kernel,sriperumbudur2010hilbert,muandet2017kernel}

\begin{remark}
The MMD provides a kernel-based notion of distance between distributions,
analogous in spirit to Wasserstein or Kullback--Leibler divergences,
and is widely used in two-sample testing and generative modeling.
It ties together information measures and Hilbert-space geometry,
completing the bridge from probability and statistics to kernel methods~\cite{gretton2012kernel,dziugaite2015mmd}
.
\end{remark}


\printglossary[type=\acronymtype, title=Acronyms]
\printglossary[type=notation,title=Notation]
\bibliographystyle{plain}
\bibliography{references}

@book{billingsley1995probability,
  author    = {Patrick Billingsley},
  title     = {Probability and Measure},
  edition   = {3},
  publisher = {John Wiley \& Sons},
  address   = {New York},
  year      = {1995}
}

@book{gut2013probability,
  author    = {Allan Gut},
  title     = {Probability: A Graduate Course},
  edition   = {2},
  series    = {Springer Texts in Statistics},
  publisher = {Springer},
  address   = {New York},
  year      = {2013}
}

@book{wasserman2004all,
  author    = {Larry Wasserman},
  title     = {All of Statistics: A Concise Course in Statistical Inference},
  series    = {Springer Texts in Statistics},
  publisher = {Springer},
  address   = {New York},
  year      = {2004}
}

@book{casella2002statistical,
  author    = {George Casella and Roger L. Berger},
  title     = {Statistical Inference},
  edition   = {2},
  publisher = {Duxbury},
  address   = {Pacific Grove, CA},
  year      = {2002}
}

@book{lehmann1998point,
  author    = {Erich L. Lehmann and George Casella},
  title     = {Theory of Point Estimation},
  edition   = {2},
  series    = {Springer Texts in Statistics},
  publisher = {Springer},
  address   = {New York},
  year      = {1998}
}

@book{cover2006elements,
  author    = {Thomas M. Cover and Joy A. Thomas},
  title     = {Elements of Information Theory},
  edition   = {2},
  publisher = {Wiley-Interscience},
  address   = {Hoboken, NJ},
  year      = {2006}
}

@article{shannon1948mathematical,
  author  = {Claude E. Shannon},
  title   = {A Mathematical Theory of Communication},
  journal = {Bell System Technical Journal},
  volume  = {27},
  pages   = {379--423, 623--656},
  year    = {1948}
}

@article{kullback1951information,
  author  = {Solomon Kullback and Richard A. Leibler},
  title   = {On Information and Sufficiency},
  journal = {Annals of Mathematical Statistics},
  volume  = {22},
  number  = {1},
  pages   = {79--86},
  year    = {1951}
}

@inproceedings{renyi1961measures,
  author    = {Alfr{\'e}d R{\'e}nyi},
  title     = {On Measures of Entropy and Information},
  booktitle = {Proceedings of the 4th {B}erkeley Symposium on Mathematical Statistics and Probability, Vol. 1},
  pages     = {547--561},
  publisher = {University of California Press},
  address   = {Berkeley, CA},
  year      = {1961}
}

@article{tsallis1988possible,
  author  = {Constantino Tsallis},
  title   = {Possible Generalization of {B}oltzmann--{G}ibbs Statistics},
  journal = {Journal of Statistical Physics},
  volume  = {52},
  number  = {1--2},
  pages   = {479--487},
  year    = {1988}
}

@book{kolmogorov1930moyenne,
  author    = {Andrey N. Kolmogorov},
  title     = {Sur la notion de la moyenne},
  publisher = {G. Bardi, Tip. della R. Accad. dei Lincei},
  address   = {Rome},
  year      = {1930}
}

@article{jizba2004generalized,
  author  = {Petr Jizba and Toshihico Arimitsu},
  title   = {Generalized Statistics: Yet Another Generalization},
  journal = {Physica A},
  volume  = {340},
  number  = {1--3},
  pages   = {110--116},
  year    = {2004}
}

@book{berlinet2004rkhs,
  author    = {Alain Berlinet and Christine Thomas-Agnan},
  title     = {Reproducing Kernel Hilbert Spaces in Probability and Statistics},
  publisher = {Kluwer Academic Publishers},
  address   = {Boston},
  year      = {2004}
}

@book{scholkopf2002learning,
  author    = {Bernhard Sch{\"o}lkopf and Alexander J. Smola},
  title     = {Learning with Kernels: Support Vector Machines, Regularization, Optimization, and Beyond},
  publisher = {MIT Press},
  address   = {Cambridge, MA},
  year      = {2002}
}

@article{fisher1925foundations,
  author  = {Ronald A. Fisher},
  title   = {Theory of Statistical Estimation},
  journal = {Proceedings of the Cambridge Philosophical Society},
  volume  = {22},
  pages   = {700--725},
  year    = {1925}
}

@book{debnath2005hilbert,
  author    = {Lokenath Debnath and Piotr Mikusi{\'n}ski},
  title     = {Introduction to {H}ilbert Spaces with Applications},
  edition   = {3},
  publisher = {Elsevier},
  address   = {Amsterdam},
  year      = {2005}
}

@book{reed1980functional,
  author    = {Michael Reed and Barry Simon},
  title     = {Functional Analysis},
  series    = {Methods of Modern Mathematical Physics},
  volume    = {1},
  edition   = {2},
  publisher = {Academic Press},
  address   = {San Diego},
  year      = {1980}
}

@book{conway2007course,
  author    = {John B. Conway},
  title     = {A Course in Functional Analysis},
  edition   = {2},
  series    = {Graduate Texts in Mathematics},
  volume    = {96},
  publisher = {Springer},
  address   = {New York},
  year      = {1990}
}

@book{ash1972probability,
  author    = {Robert B. Ash},
  title     = {Real Analysis and Probability},
  publisher = {Academic Press},
  address   = {New York},
  year      = {1972}
}

@book{bogachev2007measure,
  author    = {Vladimir I. Bogachev},
  title     = {Measure Theory},
  volume    = {1},
  publisher = {Springer},
  address   = {Berlin},
  year      = {2007}
}

@book{parthasarathy2005introduction,
  author    = {K. R. Parthasarathy},
  title     = {Introduction to Probability and Measure},
  series    = {Texts and Readings in Mathematics},
  volume    = {51},
  publisher = {Hindustan Book Agency},
  address   = {New Delhi},
  year      = {2005}
}

@book{searle1971linear,
  author    = {Shayle R. Searle},
  title     = {Linear Models},
  publisher = {Wiley},
  address   = {New York},
  year      = {1971}
}

@book{seber2003linear,
  author    = {George A. F. Seber and Alan J. Lee},
  title     = {Linear Regression Analysis},
  edition   = {2},
  series    = {Wiley Series in Probability and Statistics},
  publisher = {Wiley},
  address   = {Hoboken, NJ},
  year      = {2003}
}

@article{aitken1935least,
  author  = {Alexander C. Aitken},
  title   = {On Least Squares and Linear Combinations of Observations},
  journal = {Proceedings of the Royal Society of Edinburgh},
  volume  = {55},
  pages   = {42--48},
  year    = {1935}
}

@book{frisch1934statistical,
  author    = {Ragnar Frisch},
  title     = {Statistical Confluence Analysis by Means of Complete Regression Systems},
  publisher = {University Institute of Economics},
  address   = {Oslo},
  year      = {1934}
}

@article{mardia1970measures,
  author  = {Kanti V. Mardia},
  title   = {Measures of Multivariate Skewness and Kurtosis with Applications},
  journal = {Biometrika},
  volume  = {57},
  number  = {3},
  pages   = {519--530},
  year    = {1970}
}

@article{amaratunga2001analysis,
  author  = {Duminda Amaratunga and Javier Cabrera},
  title   = {Analysis of Data Arrays Using Metric-Space Methods},
  journal = {Journal of the American Statistical Association},
  volume  = {96},
  number  = {456},
  pages   = {1165--1174},
  year    = {2001}
}

@article{aronszajn1950reproducing,
  author  = {Aronszajn, Nachman},
  title   = {Theory of Reproducing Kernels},
  journal = {Transactions of the American Mathematical Society},
  volume  = {68},
  number  = {3},
  pages   = {337--404},
  year    = {1950},
  doi     = {10.1090/S0002-9947-1950-0051437-7}
}

@book{berlinet2004reproducing,
  author    = {Berlinet, Alain and Thomas-Agnan, Christine},
  title     = {Reproducing Kernel {H}ilbert Spaces in Probability and Statistics},
  publisher = {Kluwer Academic Publishers},
  address   = {Boston},
  year      = {2004}
}

@book{cristianini2000introduction,
  author    = {Cristianini, Nello and Shawe-Taylor, John},
  title     = {An Introduction to Support Vector Machines and Other Kernel-based Learning Methods},
  publisher = {Cambridge University Press},
  address   = {Cambridge},
  year      = {2000}
}

@article{hofmann2008kernel,
  author  = {Hofmann, Thomas and Sch{\"o}lkopf, Bernhard and Smola, Alexander J.},
  title   = {Kernel Methods in Machine Learning},
  journal = {The Annals of Statistics},
  volume  = {36},
  number  = {3},
  pages   = {1171--1220},
  year    = {2008},
  doi     = {10.1214/009053607000000677}
}

@article{parzen1962estimation,
  author  = {Parzen, Emanuel},
  title   = {On Estimation of a Probability Density Function and Mode},
  journal = {The Annals of Mathematical Statistics},
  volume  = {33},
  number  = {3},
  pages   = {1065--1076},
  year    = {1962},
  doi     = {10.1214/aoms/1177704472}
}

@article{rosenblatt1956remarks,
  author  = {Rosenblatt, Murray},
  title   = {Remarks on Some Nonparametric Estimates of a Density Function},
  journal = {The Annals of Mathematical Statistics},
  volume  = {27},
  number  = {3},
  pages   = {832--837},
  year    = {1956},
  doi     = {10.1214/aoms/1177728190}
}

@book{silverman1986density,
  author    = {Silverman, B. W.},
  title     = {Density Estimation for Statistics and Data Analysis},
  publisher = {Chapman and Hall},
  address   = {London},
  year      = {1986}
}

@inproceedings{smola2007hilbert,
  author    = {Smola, Alexander and Gretton, Arthur and Song, Le and Sch{\"o}lkopf, Bernhard},
  title     = {A {H}ilbert Space Embedding for Distributions},
  booktitle = {Algorithmic Learning Theory (ALT 2007)},
  series    = {Lecture Notes in Computer Science},
  volume    = {4754},
  pages     = {13--31},
  publisher = {Springer},
  address   = {Berlin},
  year      = {2007},
  doi       = {10.1007/978-3-540-75225-7_5}
}

@article{sriperumbudur2010hilbert,
  author  = {Sriperumbudur, Bharath K. and Gretton, Arthur and Fukumizu, Kenji and Sch{\"o}lkopf, Bernhard and Lanckriet, Gert R. G.},
  title   = {Hilbert Space Embeddings and Metrics on Probability Measures},
  journal = {Journal of Machine Learning Research},
  volume  = {11},
  pages   = {1517--1561},
  year    = {2010}
}

@inproceedings{gretton2005hsic,
  author    = {Gretton, Arthur and Bousquet, Olivier and Smola, Alexander and Sch{\"o}lkopf, Bernhard},
  title     = {Measuring Statistical Dependence with {H}ilbert--{S}chmidt Norms},
  booktitle = {Algorithmic Learning Theory (ALT 2005)},
  series    = {Lecture Notes in Computer Science},
  volume    = {3734},
  pages     = {63--77},
  publisher = {Springer},
  address   = {Berlin},
  year      = {2005},
  doi       = {10.1007/11564089_7}
}

@article{gretton2012kernel,
  author  = {Gretton, Arthur and Borgwardt, Karsten M. and Rasch, Malte J. and Sch{\"o}lkopf, Bernhard and Smola, Alexander J.},
  title   = {A Kernel Two-Sample Test},
  journal = {Journal of Machine Learning Research},
  volume  = {13},
  pages   = {723--773},
  year    = {2012}
}

@article{muandet2017kernel,
  author  = {Muandet, Krikamol and Fukumizu, Kenji and Sriperumbudur, Bharath and Sch{\"o}lkopf, Bernhard},
  title   = {Kernel Mean Embedding of Distributions: A Review and Beyond},
  journal = {Foundations and Trends in Machine Learning},
  volume  = {10},
  number  = {1--2},
  pages   = {1--141},
  year    = {2017},
  doi     = {10.1561/2200000060}
}

@book{principe2010itl,
  author    = {Principe, Jos{\'e} C.},
  title     = {Information Theoretic Learning: {R}enyi's Entropy and Kernel Perspectives},
  publisher = {Springer},
  address   = {New York},
  year      = {2010}
}

@inproceedings{dziugaite2015mmd,
  author    = {Dziugaite, Gintare Karolina and Roy, Daniel M. and Ghahramani, Zoubin},
  title     = {Training Generative Neural Networks via Maximum Mean Discrepancy Optimization},
  booktitle = {Proceedings of the Thirty-First Conference on Uncertainty in Artificial Intelligence (UAI 2015)},
  pages     = {258--267},
  year      = {2015}
}

\end{document}